\documentclass[prodmode,acmec]{ec-acmsmall} %

\usepackage[ruled]{algorithm2e}

\SetAlFnt{\small}
\SetAlCapFnt{\small}
\SetAlCapNameFnt{\small}
\SetAlCapHSkip{0pt}
\IncMargin{-\parindent}

\acmVolume{X}
\acmNumber{X}
\acmArticle{X}
\acmYear{2014}
\acmMonth{2}

\usepackage[numbers]{natbib}
\usepackage{amsmath}
\usepackage{amssymb}
\usepackage{parskip}
\usepackage{booktabs}

\newcommand{\la}{\langle}
\newcommand{\ra}{\rangle}
\newcommand{\mX}{\mathcal{X}}
\newcommand{\bR}{\mathbf{R}}
\newcommand{\grad}{\nabla}
\newcommand{\Exp}{\mathbf{E}}

\newcommand{\mM}{\ensuremath{\mathcal{M}}}
\newcommand{\hmu}{\ensuremath{\hat{\mu}}}
\newcommand{\mP}{\ensuremath{\mathcal{P}}}

\newcommand{\mF}{\mathcal{F}}
\newcommand{\mD}{\mathcal{D}}
\newcommand{\htheta}{\hat{\theta}}
\newcommand{\ttheta}{\tilde{\theta}}
\newtheorem{thm}{Theorem}[section]
\newtheorem{prop}[thm]{Proposition}

\newtheorem{ex}[thm]{Example}

\newcommand{\rmk}{\bigbreak\noindent{\bf Remark. }}
\newcommand{\unrmk}{\bigbreak}
\newcommand{\unproof}{\hfill$\Box$\bigbreak}
\newcommand{\defeq}{\stackrel{\mathrm{def}}{=}}

\newcommand{\ie} {{\it i.e. }}

\newcommand{\E}{\mathbf{E}}
\newcommand{\betavec}{\pmb{\beta}}

\newcommand{\lpf}{\mathbf{\psi}} %
\newcommand{\family}{\mathcal F} 
\newcommand{\suff}{\pmb{\phi}} %
\newcommand{\muvec}{\pmb{\mu}}
\newcommand{\st}{\hspace{10pt}\mathrm{s.t.}\hspace{10pt}}
\newcommand{\norm}[1]{|| #1 ||}
\newcommand{\tp}{\tilde{p}}
\newcommand{\CE}{C\!E}

\newcommand{\Comments}{1}
\newcommand{\mynote}[2]{\ifnum\Comments=1\textcolor{#1}{#2}\fi}

\begin{document}
\markboth{Abernethy et al.}{Information Aggregation in Exponential Family Markets}
\title{Information Aggregation in Exponential Family Markets}
\author{JACOB ABERNETHY
\affil{University of Michigan, Ann Arbor}
SINDHU KUTTY
\affil{University of Michigan, Ann Arbor}
S\'{E}BASTIEN LAHAIE
\affil{Microsoft Research, New York City}
RAHUL SAMI
\affil{Google India}}

\begin{abstract}
We consider the design of prediction market mechanisms known as automated market makers. We show that we can design these mechanisms via the mold of \emph{exponential family distributions}, a popular and well-studied probability distribution template used in statistics. We give a full development of this relationship and explore a range of benefits. We draw connections between the information aggregation of market prices and the belief aggregation of learning agents that rely on exponential family distributions. We develop a very natural analysis of the market behavior as well as the price equilibrium under the assumption that the traders exhibit risk aversion according to exponential utility. We also consider similar aspects under alternative models, such as when traders are budget constrained.
\end{abstract}




\begin{bottomstuff}
\end{bottomstuff}

\maketitle

\section{Introduction}

Prediction markets are aggregation mechanisms that allow market prices to be interpreted as predictive probabilities on an event. Each trader in the market is assumed to have some private information that he uses to make a prediction on the outcome of the event. Traders are allowed to report their beliefs by buying and selling securities whose ultimate payoff depends on the future outcome. This will affect the state of the market, thus updating the predictive probabilities for the event. Further, since the trades are done sequentially, the trader is allowed to observe all past trades in the market and update his private information based on this information. In this sense the market prices, which are in effect the prices at which the marginal trader is willing to buy or sell the available securities, can be interpreted as an aggregate ``consensus probability forecast'' of the event in question.
Much of the work on prediction market design has focused heavily on structural properties of the mechanism: incentive compatibility, the market maker loss, the available liquidity, the fluctuations of the prices as a function of the trading volume, to name a few. Absent from much of the literature is a corresponding \emph{semantics} of the market behavior or the observed prices. That is, how can we interpret the equilibrium market state when we have a number of traders with diverse beliefs on the underlying state of the world? In what sense is the market an aggregation mechanism? Do price changes relate to our usual Bayesian notion of information incorporation via posterior updating?

In the present work we show that a number of classical statistical tools can be leveraged to design a prediction market framework in the mold of \emph{exponential family distributions}; we show that this statistical framework leads to a number of attractive properties and interpretations. Common concepts in statistics---including \emph{entropy maximization}, \emph{log loss}, and \emph{bayesian inference}---relate to natural aspects of our class of mechanisms. In particular, the central objects in our market framework can be interpreted via concepts used to define exponential families:
\begin{itemize}
	\item the market's \emph{payoff function} corresponds to the \emph{sufficient statistics} of the distribution;
	\item the vector of \emph{outstanding shares} in the market corresponds to the \emph{natural parameter vector} of the distribution;
	\item the market prices correspond to \emph{mean parameters};
	\item the market's cost function corresponds to the distribution's \emph{log-partition} function.
\end{itemize}

We begin in Section~\ref{sec:scoring} with a discussion of scoring rules based on exponential family distributions, and we show how the framework leads to a variety of scoring rules for continuous outcome spaces.
We turn our attention to market design in Section~\ref{sec:maxentmarkets} and give a full description of our proposed mechanisms. In addition to showing the syntactic relationship between exponential families and prediction markets, we explore a number of rich semantic implications as well. In particular, we show that our formulation allows us to analyze the evolution of the market under various models of trader behavior:
\begin{itemize}
	\item Trader behavior varies depending on how they assimilate information; for example, should we consider our agents as Bayesians or frequentists. In Section~\ref{sec:bayesians} we consider traders that use a conjugate prior to update their beliefs, and we study how their trades would affect the market state.
	\item In Section~\ref{sec:exputil} we consider \emph{risk-averse} agents that optimize their bets according to exponential utility. 
	In this case we can characterize precisely how a single trader interacts with the market, as well as the equilibrium reached given multiple traders; this result is achieved via a potential game argument. The eventual market state is a weighted combination of traders' beliefs and the initial state; the weights are proportional to risk aversion parameters.
	\item In Section~\ref{sec:budgets} we consider \emph{budget-limited traders} who are constrained in how much they influence the market. We analyze the market  under these circumstances; we are able to show that traders with good information can expect to profit and their influence over the market state increases over time whereas malicious traders have limited impact on the market.
\end{itemize}

\textbf{Related Work.}
The notion of an exponential family distribution is fundamental to this paper. For comprehensive introductions to these distributions, see~\citep{Barndorff78,WainJordan08}. Exponential families are intimately tied to the notions of log loss and entropy, but can be generalized to other types of convex losses and information, as shown by~\citet{GrunwaldDawid}, who also make a connection to scoring rules.

Scoring rules are a measure of prediction accuracy, and we are concerned here with scoring rules for statistic expectations, typically over infinite outcome spaces. Such rules have been characterized by~\citet{Savage71}; see also~\citep{frongillo2013elicitingmeans,lambert2008eliciting}. Our rules are of course special cases of this characterization, but it appears the range of elegant scoring rules that arise from exponential families has not been appreciated. Indeed,~\citet{gneiting2007strictly} observe that specific instances of scoring rules for continuous outcomes are lacking, and survey various possibilities.

In a seminal paper,~\citet{H03} showed how to form a prediction market based on a sequentially-shared scoring rule, and specifically proposed the logarithmic market scoring rule (LMSR) based on log loss for finite outcome spaces~\citep{H07}. The markets we introduce are direct generalizations of the LMSR to continuous outcomes, but take the form of cost-function based markets as introduced by~\citet{chen2007utility}.~\citet{Gao09} and~\citet{measurable13} also consider extending various market makers to infinite outcome spaces.

Prediction markets are known to perform well in practice~\citep{AGT07,pennock2001real}. However, a sound theory for interpreting trader behavior and market prices is an ongoing area of study~\citep{wolfers2006interpreting}. At one extreme, agents are assumed myopic and risk-neutral, implying they move the market state to their belief~\citep{chen2010new}. At the other extreme, agents are strategic and the market fully incorporates all information~\citep{ostrovsky2012}. 

We are not aware of any works that consider risk-averse agents within cost-function based markets. However, risk aversion is a fundamental component of mathematical finance and portfolio optimization, and there are close connections between the notion of a cost function and that of a convex risk measure~\citep{follmerRisk,follmer2011entropic}. Indeed, they arise from the same axioms as noted by~\citet{othman2011liqu}. We see the potential to draw more on the mathematical finance literature to take into account risk aversion, as prediction markets are simply single-period financial markets~\citep[Part I]{follmerFinance}. We also note that connections between Machine Learning and market mechanisms have been explored in \cite{storkey11}.

\section{Generalized Log Scoring Rules}
\label{sec:scoring}

We consider a measurable space consisting of a set of outcomes $\mX$ together with a $\sigma$-algebra $\mF$. An agent or expert has a \emph{belief} over potential outcomes taking the form of a probability measure absolutely continuous with respect to a base measure~$\nu$.\footnote{Recall that a measure $P$ is absolutely continuous with respect to~$\nu$ if $P(A) = 0$ for every $A \in \mF$ for which $\nu(A) = 0$. In essence the base measure~$\nu$ restricts the support of $P$. In our examples $\nu$ will typically be a restriction of the Lebesgue measure for continuous outcomes or the counting measure for discrete outcomes.} Throughout we represent the belief as the corresponding density $p$ with respect to~$\nu$. Let $\mP$ denote the set of all such probability densities.

We are interested in eliciting information about the agent's belief, in particular expectation information. Let $\phi: \mX \rightarrow \bR^d$ be a vector-valued random variable or \emph{statistic}, where $d$ is finite. The aim is to elicit $\mu = \E_{p}[\phi(x)]$ where $x$ is the random outcome. A \emph{scoring rule} is a device for this purpose. Let $$\mM = \left\{ \mu \in \bR^d : \Exp_p[\phi(x)] = \mu,\, \mbox{for some $p \in \mP$} \right\}$$ be the set of realizable statistic expectations.  A scoring rule $S : \mM \times \mX \rightarrow \bR \cup \{-\infty\}$ pays the agent $S(\hmu,x)$ according to how well its report $\hmu \in \mM$ agrees with the eventual outcome $x \in \mX$. The following definition is due to~\citet{lambert2008eliciting}.
\begin{definition} \label{def:proper-score}
A scoring rule $S$ is  \emph{proper} for statistic $\phi$ if for each $\mu \in \mM$ and $p \in \mP$ with expected statistic $\mu$, we have for all $\hmu \neq \mu$
\begin{equation} \label{eq-proper}
\Exp_p[S(\mu,x)] \geq \Exp_p[S(\hmu,x)].
\end{equation}
\end{definition}

\noindent
Given a proper scoring rule $S$ any affine transformation $\tilde{S}(\mu,x) = aS(\mu,x) + b(x)$ of the rule, with $a > 0$ and $b$ an arbitrary real-valued function of the outcomes, again yields a proper scoring rule termed \emph{equivalent}~\citep{Dawid98,gneiting2007strictly}. Throughout we will implicitly apply such affine transformations to obtain the clearest version of the scoring rule. We will also focus on scoring rules where inequality~(\ref{eq-proper}) is strict to avoid trivial cases such as constant scoring rules.

 Classically, scoring rules take in the entire density $p$ rather than just some statistic, and incentive compatibility must hold over all of $\mP$. When the outcome space is large or infinite, it is not feasible to directly communicate $p$, so the definition allows for summary information of the belief.

Note that Definition~\ref{def:proper-score} places only mild information requirements on the part of the agent to ensure truthful reporting. Because condition~(\ref{eq-proper}) holds for all $p$ consistent with expectation $\mu$, it is enough for the agent to simply know the latter and not the complete density to be properly incentivized. However, the agent must also agree with the support of the density as implicitly defined by base measure $\nu$. 

When the outcome space is finite we recover classical scoring rules by using the statistic $\phi : \mX \rightarrow \{0,1\}^{\mX}$ that maps an outcome $x$ to a unit vector with a 1 in the component corresponding to $x$. The expectation of $\phi$ is then exactly the probability mass function.

\subsection{Proper Scoring from Maximum Entropy}

Our starting point for designing proper scoring rules is the classic logarithmic scoring rule for eliciting probabilities in the case of finite outcomes. This rule is simply $S(p,x) = \log p(x)$, namely we take the log likelihood of the reported density at the eventual outcome. To generalize the rule to expected statistics rather than full densities, we consider a subset of densities $\mD \subseteq \mP$. If there is a bijection between the sets $\mD$ and $\mM$, then we say that $\mM$ parametrizes $\mD$ and write $p(\cdot\,;\mu)$ for the density mapping to $\mu$. Given such a family parametrized by the relevant statistics, the generalized log scoring rule is then
\begin{equation} \label{log-score}
S(\mu, x) = \log p(x;\mu).
\end{equation}
Even though the log score is only applied to densities from $\mD$, according to Definition~\ref{def:proper-score} it must work over all densities in $\mP$. It turns out this is possible if $\mD$ is chosen appropriately, drawing on a well-known duality between maximum likelihood and maximum entropy~\citep{GrunwaldDawid}.

\subsubsection*{Exponential Families}

We let $p(x;\mu)$ be the maximum entropy distribution with expected statistic $\mu$. Specifically, it is the solution to the following program:\footnote{We assume that the minimum is finite and achieved for all $\mu \in \mM$. Some care is needed to ensure this holds for specific statistics and outcome spaces. For example, taking outcomes to be the real numbers, there is no maximum entropy distribution with a given mean $\mu$ (one can take densities tending towards the uniform distribution over the reals), but there is always a solution if we constrain both the mean and variance.}
\begin{equation} \label{maxent-prog}
 \min_{p \in \mP} \:F(p) \st \Exp_{p}[\phi(x)] = \mu,
\end{equation}
where the objective function is the negative entropy of the distribution, namely
\[ F(p) = \int_{x \in \mX} p(x) \log p(x)\, d\nu(x).
\]
Note that the explicit set of constraints in~(\ref{maxent-prog}) are linear, whereas the objective is convex. We let $G : \mM \rightarrow \bR$ be the optimal value function of~(\ref{maxent-prog}), meaning $G(\mu)$ is the negative entropy of the maximum entropy distribution with expected statistics $\mu$. 

It is well-known that solutions to~(\ref{maxent-prog}) are \emph{exponential family} distributions, whose densities with respect to $\nu$ take the form
\begin{equation} \label{exp-fam}
p(x;\theta) = \exp( \la \theta, \phi(x) \ra - T(\theta) ).
\end{equation}
The density is stated here in terms of its \emph{natural} parametrization $\theta \in \bR^d$, where $\theta$ arises as the Lagrange multiplier associated with the linear constraints in~(\ref{maxent-prog}). The term $T(\theta)$ essentially arises as the multiplier for the normalization constraint (the density must integrate to 1), and so ensures that~(\ref{exp-fam}) is normalized:
\begin{equation} \label{log-part}
T(\theta) = \log \int_{\mX} \exp \la \theta, \phi(x) \ra \,d\nu(x).
\end{equation}
The function $T$ is known as the \emph{log-partition} or \emph{cumulant} function corresponding to the exponential family. Its domain is $\Theta = \{ \theta \in \bR^d : T(\theta) < +\infty \}$, called the natural parameter space. The exponential family is \emph{regular} if $\Theta$ is open---almost all exponential families of interest, and all those we consider in this work, are regular. The family is \emph{minimal} if there is no $\alpha \in \Theta$ such that $\la \alpha, \phi(x) \ra$ is a constant over $\mX$ ($\nu$-almost everywhere); minimality is a property of the associated statistic $\phi$, usually called the \emph{sufficient statistic} in the literature. 

The following proposition collects the relevant results on regular exponential families; proofs may be found in~\citet[Prop.\ 3.1--3.2, Thm.\ 3.3--3.4]{WainJordan08} and see also~\citet[Lem.\ 1, Thm.\ 2]{BanerjeeDhGh05b}. A convex function $T$ is of \emph{Legendre type} if it is proper, closed, strictly convex and differentiable on the interior of its domain, and $\lim_{\theta \rightarrow \bar{\theta}} \norm{\grad T(\theta)} = +\infty$ when $\bar{\theta}$ lies on the boundary of the domain. 
\begin{prop}\label{prop:exp}
Consider a regular exponential family with minimal sufficient statistic. The following properties hold:
\begin{enumerate}
\item $T$ and $G$ are of Legendre type, and $T = G^*$ (equivalently $G = T^*$). 
\item The gradient map $\grad T$ is one-to-one and onto the interior of $\mM$. Its inverse is $\grad G$ which is one-to-one and onto the interior of $\Theta$.
\item The exponential family distribution with natural parameter $\theta \in \Theta$ has expected statistic $\mu = \Exp_p[\phi(x)] = \grad T(\theta)$.
\item The maximum entropy distribution with expected statistic $\mu$ is the exponential family distribution with natural parameter $\theta = \grad G(\mu)$. 
\end{enumerate}
\end{prop}
In the above $T^*$ denotes the convex conjugate of $T$, which here can be evaluated as $T^*(\mu) = \sup_{\theta \in \Theta} \la \theta, \mu \ra - T(\theta)$. Similarly, $G^*(\theta) = \sup_{\mu \in \mM} \la \theta, \mu \ra - G(\mu)$.

\subsubsection*{Proper Log Scoring}

We are now in a position to analyze the log scoring rule under exponential family distributions. From our discussion so far, we have that an exponential family density can be parametrized either by the natural parameter $\theta$, or by the mean parameter $\mu$, and that the two are related by the invertible gradient map $\mu = \grad T(\theta)$. We will write $p(x;\theta)$ or $p(x;\mu)$ given the parametrization used.

The following observation is crucial. Let $\tp \in \mP$ be a density (not necessarily from an exponential family) with expected statistic $\mu$, let $p(\cdot\,;\mu)$ be the exponential family with the same expected statistic, and let $\hmu \in \mM$ be an alternative report. Then from~(\ref{exp-fam}) note
\begin{equation} \label{equalizer-rule}
\Exp_{\tp}[\log p(x;\hmu)] = \Exp_{p(\cdot;\mu)}[\log p(x;\hmu)] = \la \htheta, \mu \ra - T(\htheta),
\end{equation}
where $\htheta = \grad G(\hmu)$ is the natural parameter for the exponential family with statistic $\hmu$. We see from this that the expected log score only depends on the expectation $\mu$ of the underlying density, not the full density, which is how we can achieve proper scoring according to Definition~{\ref{def:proper-score}.
\begin{theorem} \label{maxent-score}
Consider the logarithmic scoring rule $S(\mu,x) = \log p(x;\mu)$ defined over a set of densities $\mD$ parametrized by $\mM$. The scoring rule is proper if and only if $\mD$ is the exponential family with statistic $\phi$. 
\end{theorem}
\begin{proof}
Let $\mu,\hmu \in \mM$ be the agent's true belief and an alternative report, and let $p \in \mP$ be a density consistent with $\mu$. Let $\theta = \grad G(\mu)$ and $\htheta = \grad G(\hmu)$, and note that $\mu = \grad T(\theta)$. We have
\begin{eqnarray}
\Exp_p[\log p(x;\mu)] - \Exp_p[\log p(x;\hmu)]
& = & \la \theta, \mu \ra - T(\theta) - \la \htheta, \mu \ra + T(\htheta) \nonumber\\
& = & T(\htheta) - T(\theta) - \la \htheta - \theta, \mu \ra \nonumber\\
& = & T(\htheta) - T(\theta) - \la \htheta - \theta, \grad T(\theta) \ra. \label{breg-div}
\end{eqnarray}
The latter is positive by the strict convexity of $T$, which shows that the log score is proper. For the converse, assume the defined log score is proper. By the Savage characterization of proper scoring rules for expectations (see~\citet[Thm.\ 1]{gneiting2007strictly} and~\citet{Savage71}), we must have
$$
S(\mu,x) = G(\mu) - \la \grad G(\mu), \mu - \phi(x) \ra
$$
for some strictly convex function $G$. Let $T = G^*$, so that $\grad G = {\grad T}^{-1}$, and let $\theta = \grad G(\mu)$. Then the above can be written as
\begin{eqnarray*}
\log p(x;\mu) & = & G(\mu) - \la \grad G(\mu), \mu - \phi(x) \ra \\
& = & \la \theta, \mu \ra - T(\theta) - \la \theta, \mu - \phi(x) \ra = \la \theta, \phi(x) \ra - T(\theta),
\end{eqnarray*}
which shows that $p(x;\mu)$ takes the form of an exponential family.
\end{proof}

As further intuition for the result, note that~(\ref{breg-div}) is the definition of the `Bregman divergence' with respect to strictly convex function $T$, written $D_T$. Therefore we have 
\begin{equation*}
\Exp_p[\log p(x;\mu)] - \Exp_p[\log p(x;\hmu)] = D_T(\htheta,\theta) = D_G(\mu, \hmu),
\end{equation*}
where the last equality is a well-known identity relating the Bregman divergences of $T$ and $T^* = G$. The equation states that the agent's regret from misreporting its mean parameter does not depend on the full density $p$, only the mean $\mu$.   

\subsection{Examples: Moments over the Real Line}

Theorem~\ref{maxent-score} leads to a straightforward procedure for constructing score rules for expectations. Define the relevant statistic, and consider the maximum entropy (equivalently, exponential family) distribution consistent with the agent's reported mean $\mu$. The scoring rule compensates the agent according to the log likelihood of the eventual outcome according to this distribution. The interpretation is that the agent is only providing partial information about the underlying density, so the principal first infers a full density according to the principle of maximum entropy, and then scores the agent using the usual log score.
  
An advantage of this generalization of the log score is that, for many domains (multi-dimensional included) and expectations of interest, it leads to novel closed-form scoring rules. By examining the log densities of various exponential families, we can for instance obtain scoring rules for several different combinations of the arithmetic, geometric, and harmonic means, as well as higher order moments. The following examples illustrate the construction.
\begin{example} \label{ex:exponential}
As base measure we take the Lebesgue restricted to $[0,+\infty)$, and we consider the statistic $\phi(x) = x$ so that we are simply eliciting the mean. The maximum entropy distribution with a given mean $\mu$ is the exponential distribution, and taking its log density gives the scoring rule
\begin{equation} \label{exp-scoring}
S(\mu, x) = -\frac{x}{\mu} - \log \mu.
\end{equation}
We stress that although this rule is derived from the exponential distribution, Theorem~\ref{maxent-score} implies that it elicits the mean of any distribution supported on the non-negative reals (e.g., Pareto, lognormal). Indeed, it is easy to see that the expected score~(\ref{exp-scoring}) depends only on the mean of the agent's belief because it is linear in $x$. 
As a generalization of this example, the maximum entropy distribution for the $k$-th moment $\phi(x) = x^k$ with respect to the same base measure is the Weibull distribution. Taking its log density leads to the following equivalent scoring rule:
\begin{equation} \label{weibull-scoring}
S(\mu, x) = (k-1) \log x - k \log \mu - \Gamma\left(1+\frac{1}{k}\right)^k \left(\frac{x}{\mu}\right)^k,
\end{equation}
where $\Gamma$ denotes the gamma function (the extension of the factorial to the reals). We have not found either scoring rule~(\ref{exp-scoring}) or~(\ref{weibull-scoring}) in the literature.
\end{example}
\begin{example} \label{ex:gaussian}
As a base measure we take the Lebesgue over the real numbers $\bR$. We are interested in eliciting the mean $\mu$ and variance $\sigma^2$, so as a statistic we take $\phi(x) = (x, x^2)$ for which $\Exp_p[\phi(x)] = (\mu, \mu^2+\sigma^2)$. The max entropy distribution for a given mean and variance is the Gaussian, whose log density gives the scoring rule
\begin{equation}
S((\mu, \sigma^2), x) = -\frac{(x-\mu)^2}{\sigma^2} - \log \sigma^2.
\end{equation}
Again, we stress that this scoring rule elicits the mean and variance of any density over the real numbers, not just those of a normal distribution. The construction easily generalizes to a multi-dimensional outcome space by taking the log density of the multivariate normal:
\begin{equation}
S((\mu, \Sigma), x) = -(x - \mu)'\Sigma^{-1}(x - \mu) - \log |\Sigma|.
\end{equation}
Here the statistics being elicited are the mean vector $\mu$ and the covariance matrix $\Sigma$. These scoring rules have been studied by~\citet{dawid1999coherent} as rules that only depend on the mean and variance of the reported density. They note that these rules are weakly proper (because they do not distinguish between densities with the same first and second moments), but do not make the point that knowledge of the full density is not necessary on the part of the agent.
\end{example}

In the above, Example~\ref{ex:gaussian} illustrates an important point about parametrizations of the elicited expectations. The variance $\sigma^2$ cannot be written as $\Exp[\phi(x)]$ for any $\phi$, because the mean $\mu$ enters the definition of $\sigma^2$ but is not available when $\phi$ is defined (indeed it is elicited in tandem with the variance).\footnote{This is an intuitive but far from formal explanation for the fact that the dimension of the message space, or \emph{elicitation complexity}, for eliciting the variance is at least 2~\citep{lambert2008eliciting}.} Instead one must use the first two \emph{uncentered} moments $\Exp[x]$ and $\Exp[x^2]$. These are in bijection with $\mu$ and $\sigma^2$, so the resulting scoring rule can be re-written in terms of the latter. Therefore, it is possible to elicit not just expectations but also bijective transformations of expectations. 

\section{Exponential Family Markets}
\label{sec:maxentmarkets}

In a single-agent setting, a scoring rule is used to \emph{elicit} the agent's belief. In a multi-agent setting, a prediction market can be used to \emph{aggregate} the beliefs of the agents. In his seminal paper~\citet{H03} introduced the idea of a market scoring rule, which inherits the appealing elicitation and aggregation properties of both in order to perform well in thin or thick markets. In this section, we adapt the generalized log scoring rule to a market scoring rule which leads to markets with simple closed-form cost functions for many statistics of interest.%

\subsection{Prediction Market}
\label{sec:pred-market}

In a prediction market an agent's expected belief $\mu$ is elicited indirectly through the purchase and sale of contingent claim securities. Under this approach, each component~$i$ of the statistic $\phi$ is interpreted as the payoff function of a security; that is, a single share of security $i$ pays off $\phi_i(x)$ when $x \in \mX$ occurs. Thus if the portfolio of shares held by the agent is $\delta \in \bR^d$, where entry $\delta_i$ corresponds to the number of shares of security $i$, then the payoff to the agent when $x$ occurs is evaluated by taking the inner product $\la \delta, \phi(x) \ra$. 

As a concrete example, recall that in the classic finite-outcome case the statistic has a component for each outcome $x$ such that $\phi_x(x') = 1$ if $x' = x$ and 0 otherwise. Therefore the corresponding security pays 1 dollar if outcome $x$ occurs. (These are known as Arrow-Debreu securities.) In Example~\ref{ex:exponential} the one-dimensional statistic is $\phi(x) = x$, corresponding to a security whose payoff is linear in the outcome $x \in \bR_+$. (This amounts to a futures contract.) 

The standard way to implement a prediction market in the literature, due to~\citet{chen2007utility}, is via a centralized market maker. The market maker maintains a convex, differentiable cost function $C : \bR^d \rightarrow (-\infty,+\infty]$, where $C(\theta)$ records the revenue collected when the vector of outstanding shares is $\theta$. The cost to an agent of purchasing portfolio $\delta$ under a market state of $\theta$ is $C(\theta + \delta) - C(\theta)$, and therefore the instantaneous prices of the securities are given by the gradient $\grad C(\theta)$. 

A risk-neutral agent will choose to acquire shares up to the point where, for each share, expected payoff equals marginal price. Formally, if the agent acquires portfolio $\delta$, moving the market state vector to $\theta' = \theta + \delta$, then we must have
\begin{equation} \label{risk-neutral-agent}
\Exp_p[\phi(x)] = \grad C(\theta').
\end{equation}
In this way, by its choice of $\delta$, the agent reveals that its expected belief is $\mu = \grad C(\theta')$. We stress that this observation relies on the assumptions that 1) the agent is risk-neutral, 2) the agent does not incorporate the market's information into its own beliefs, and 3) the agent is not budget constrained. We will examine relaxations of each assumption in later sections.

\subsection{Information-Theoretic Interpretation}

In the remainder of this paper we focus on the following cost function, which arises from the ``generalized'' logarithmic market scoring rule (LMSR):
\begin{equation} \label{lmsr-cost}
C(\theta) = \log \int_{x \in \mX} \exp \la \theta, \phi(x) \ra \,d\nu(x).
\end{equation}
This is of course exactly the log-partition function~(\ref{log-part}) for the exponential family with sufficient statistic $\phi$, and we recover the classic LMSR using outcome indicator vectors as statistics. Because an agent would never select a portfolio with infinite cost, the effective domain (i.e., the possible vectors of outstanding shares) of $C$ is $\Theta = \{\theta \in \bR^d : C(\theta) < +\infty \}$, which gives an economic interpretation to the natural parameter space of an exponential family.

The correspondence between the cost function~({\ref{lmsr-cost}) and the log-partition function~(\ref{log-part}) suggests the following interpretation. The market maker maintains an exponential family distribution over the state space $\mX$ parametrized by share vectors that lie in $\Theta$. When an agent buys shares, it moves the distribution's natural parameter so that the market prices matches its beliefs, or in other words the market's mean parametrization matches the agent's expectation. 

There is a well-known duality between scoring rules and cost-function based markets~\citep{acv13,H03}. To see this in our context, recall from~(\ref{equalizer-rule}):
\begin{equation*}
\Exp_{\tp}[\log p(x;\hmu)] = \la \htheta, \mu \ra - T(\htheta)
\end{equation*}
where $\tp$ is the agent's belief and $\hmu$ the agent's report. The expected log score from reporting $\hmu$ is exactly the same as the expected payoff from buying portfolio of shares $\htheta = \grad C(\hmu)$ (assuming an initial market state of 0), as $\la \htheta, \mu \ra$ is the expected revenue and $T(\htheta)$ is the cost. As in Section~\ref{sec:scoring} this reasoning relies on the assumption of risk-neutrality, not on any specific form for the agent's belief.

The agent's expected profit from moving the share vector from $\theta$ to $\theta'$ is
\begin{eqnarray*}
& & \la \theta' - \theta, \mu \ra - C(\theta') + C(\theta) \\
& = & C(\theta) - C(\theta') - \la \theta - \theta', \grad C(\theta) \rangle \\
& = & D_C(\theta, \theta') = D_{C^*}(\mu', \mu),
\end{eqnarray*}
recalling~(\ref{breg-div}). Now~\citet{BanerjeeDhGh05} have observed (among others) that the Kullback-Leibler divergence between two exponential family distributions is the Bregman divergence, with repect to the log-partition function, between their natural parameters. The agent's expected profit is therefore the KL divergence between the market's implied expectation and the exponential family corresponding to the agent's expectation, a well-known property from the classical LMSR~\citep{H07}.

\subsection{Examples: Real Line and the Sphere}

Let us now revisit our scoring rule examples from Section~\ref{sec:scoring} in the context of prediction markets. The relevant entities now are the payoff function, the effective domain of shares, and the cost function.
\begin{example} \label{ex:exp-market}
We consider outcomes over the positive reals $\bR_+$ and set up a market for the expected outcome, consisting of a single security that pays off $\phi(x) = x$. The log partition function of the exponential distribution leads to the following cost function:
$$
C(\theta) = -\log(-\theta).
$$
The effective domain is $\Theta = \{\theta \in \bR : \theta < 0\}$.  This means the market must start with a negative number of outstanding shares for the security, and the number of shares must stay negative. The market maker need not explicitly enforce this, because by the Legendre property of $C$ the cost tends to $+\infty$ as the outstanding shares approach the boundary, which is straightforward to see in this example.
\end{example}
\begin{example} \label{ex:gauss-market}
We consider outcomes over the real line $\bR$ and set up a market with securities corresponding to the first two uncentered moments (i.e, agents are betting on the return and volatility). The securities are defined by the payoffs $\phi(x) = (x, x^2)$. The log partition function of the normal distribution, under its natural parametrization, leads to the following cost function:
$$ 
C(\theta) = -\frac{\theta_1^2}{4\theta_2} - \frac{1}{2}\log(-2\theta_2).
$$
The effective domain is $\Theta = \{(\theta_1,\theta_2) \in \bR^2 : \theta_2 < 0 \}$. Again, we have here an instance where it is not possible for the number of outstanding shares of the second security to exceed 0. However, an arbitrary amount of the securities can be sold short, which corresponds to increasing the variance of the market's estimate.
\end{example}
\begin{example} \label{ex:vonmises-market}
As another example let the outcome space be the $d$-dimensional unit sphere. This setting was considered by~\citet{acv13} who provide a cost function implicitly defined through a variational characterization. The maximum entropy approach leads to another alternative. We have a security for each of the $d$ dimensions, and security $i$ simply pays off $\phi_i(x) = x_i$, where $x \in \bR^d$ is the unit-norm outcome. The maximum entropy distribution over the sphere with such sufficient statistics is the von Mises-Fisher distribution. The log partition function corresponds to
$$
C(\theta) = I_{\frac{d}{2}-1}(\norm{\theta}) - \left(\frac{d}{2} - 1 \right)\log \norm{\theta},
$$
where $I_r$ refers to the modified Bessel function of first kind and order $r$; see~\citet{BanerjeeDhGh05} for an explanation of these quantities. The effective domain of $\theta$ is the positive orthant in $\bR^d$. The mean parametrization of the von Mises-Fisher distribution gives a generalized log scoring rule for the expected outcome components, but it is unwieldy and involves several special functions. 
\end{example}

\section{Bayesian Traders with Linear Utility} \label{sec:bayesians}

In the standard model of cost-function based prediction markets, a sequence of myopic, risk-neutral agents arrive and trade in the market~\citep{chen2007utility,chen2010new}. As we saw in Section~\ref{sec:pred-market}, such a trader moves the prices to its own expectation $\mu$. However, this means that the market does not perform meaningful aggregation of agents' beliefs, as the final prices are simply the final agent's expectation.

In this section we examine the aggregation behavior of the market when agents are Bayesian and take into account the current market state when forming their beliefs. This requires more structure to their beliefs. For this section and the remainder of the paper, we will assume that agents have \emph{exponential family beliefs}.

The exponential families framework is well-suited to reasoning about Bayesian updates. As before let the data distribution be given by $p(x;\theta)=\exp(\la\theta,\phi(x)\ra-T(\theta))$ where $T$ is the log partition function and $\phi$ are the sufficient statistics. Instead of direct beliefs about the data distribution the agent maintains a conjugate prior over the parameters $\theta$. Every exponential family admits a conjugate prior of the form
$$p(\theta;b_{0})=\exp(\la n\nu,\theta\ra + nT(\theta) - \psi(\nu,n)).$$
Note that this is also an exponential family with natural parameter $b_0 = (n\nu,n)$ where $\nu \in \bR^d$ and $n$ is a positive integer. The sufficient statistic maps $\theta$ to $(\theta, T(\theta))$, and the log partition function $\psi$ is defined as the normalizer as usual. For a complete treatment of exponential families conjugate priors, see for instance~\citet{Barndorff78}. Now~\citet[Thm.\ 2]{diaconis79} and~\citet{jewell1974credible} have shown that
\begin{equation}\label{eq:prior}
\E_{\theta\sim b_{0}}\E_{x\sim \theta}[\phi(x)]=\nu,
\end{equation}
meaning that $\nu = n \nu / n$ is the posterior mean. Thus, it is helpful to think of the prior as being based on a `phantom' sample of size $n$ and mean $\nu$. Suppose now that the agent observes an empirical sample with mean $\hmu$ and size $m$. By a standard derivation~\citep[see][]{diaconis79}, the posterior conjugate prior parameters become $n\nu \leftarrow n\nu + m\hmu$ and $n \leftarrow n+m$, and the posterior expectation~(\ref{eq:prior}) evaluates to
\begin{equation} \label{posterior}
\frac{n\nu + m\hmu}{n+m}.
\end{equation}
Thus the posterior mean is a convex combination of the prior and empirical means, and their relative weights depend on the phantom and empirical sample sizes.%

Consider Bayesian agents maintaining an exponential family conjugate prior over the data model's natural parameters (equivalently, the expected security payoffs). Each agent has access to a private sample of the data of size $m$ with mean statistic $\hmu$. If $n$ agents have arrived before to trade, then the current market prices $\mu$ correspond to the phantom sample, and the phantom sample size is $nm$. After forming the posterior~(\ref{posterior}) with these substitutions, the (risk-neutral) agent purchases shares $\delta$ to move the current market share vector to
$$
\grad C(\theta + \delta) = \frac{n\nu + \hmu}{n+1}.
$$
As a result, the final market prices under this behavior are a simple average of the agent's mean parameters and the initial market prices. We note that to facilitate such belief updating, the market should post the number of trades since initialization.

\section{Risk-Averse Traders with Exponential Utility} \label{sec:exputil}

In this section we relax the standard assumption that agents in the market are risk-neutral. We show that with sufficient extra structure to the agents' beliefs and utilities, the market performs a clean aggregation of the agents' beliefs via a simple weighted average.
Assume that the agent has an exponential utility function for wealth $w$:
\begin{equation} \label{exp-util}
U_a(w) = -\frac{1}{a} \exp(-aw).
\end{equation}
Here $a$ controls the risk aversion: the agent's aversion grows as $a$ increases, and as $a$ tends to 0 we approach linear utility (risk-neutrality). Specifically, $a$ is the Arrow-Pratt coefficient of absolute risk aversion, and exponential utilities of the form~(\ref{exp-util}) are the unique utilities that exhibit constant absolute risk aversion~\citep[Chap.\ 11]{VarianMicro}. 

If wealth is distributed according to a probability measure $P$, then the \emph{certainty equivalent} of a random amount of wealth is defined as
$$
\CE(w) = U_a^{-1} (\Exp_P\left[ U_a(w) \right]).
$$
Suppose as before that the agent's belief over outcomes takes the form of a density $p$ with respect to base measure $\nu$. There is a close relationship between the log-partition function and the certainty equivalent under exponential utility~\citep[see][]{ben2007old}.
\begin{lemma} \label{lem:exp-util}
The certainty equivalent of the agent's expected profit, with exponential utility, when acquiring shares $\delta$ under a market state of $\theta$ is
\begin{equation}
\log a - T^p(-a\delta) - aC(\theta+\delta) + aC(\theta),
\end{equation}
where $T^p$ is the log partition function~(\ref{log-part}) with a base measure of $p\,d\nu$. Furthermore, if the agent's belief is an exponential family with natural parameter $\htheta$, we have
\begin{equation*}
T^p(\delta) = T(\htheta + \delta) - T(\htheta),
\end{equation*}
where $T$ is the usual log partition function with base measure $\nu$.
\end{lemma}
\begin{proof}
Explicitly, the certainty equivalent of the profit is
\begin{eqnarray*}
& & \CE(\,\la\delta,\phi(x)\ra - [C(\theta+\delta) - C(\theta)]\,)\\
& = & -\log \int_{\mX} \frac{1}{a} \exp \left(\la-a\delta,\phi(x)\ra + a [C(\theta+\delta) - C(\theta)]\right) p(x)d\nu(x)\\
& = & \log a - a [C(\theta+\delta) - C(\theta)] - \log \int_{\mX} \exp\la-a\delta,\phi(x)\ra\, p(x)\,d\nu(x)\\
& = & \log a - a [C(\theta+\delta) - C(\theta)] - T^p(-a\delta).
\end{eqnarray*}
For the second part of the result, we have
\begin{eqnarray*}
T^p(\delta) & = & \log \int_{\mX} \exp\la\delta,\phi(x)\ra\,p(x;\htheta)\, d\nu(x)\\
& = &  \log \int_{\mX} \exp( \la\delta+\htheta,\phi(x)\ra - T(\htheta) )\, d\nu(x)\\
& = &  T(\htheta+\delta) - T(\htheta) + \log \int_{\mX} \exp( \la\delta+\htheta,\phi(x)\ra - T(\htheta+\delta) )\, d\nu(x)\\
& = &  T(\htheta+\delta) - T(\htheta) + \log \int_{\mX} p(x;\htheta+\delta) \, d\nu(x) =   T(\htheta+\delta) - T(\htheta),
\end{eqnarray*}
where the last line follows from the fact that density $p(x;\htheta+\delta)$ integrates to 1.
\end{proof}

\noindent
Recall that for the generalized LMSR, the cost function $C$ is exactly the log partition function $T$. We are therefore lead to the following understanding of a risk-averse agent's behavior in such a market.
\begin{theorem} \label{thm:exp-util}
Suppose an agent has exponential utility with coefficient $a$ and exponential family beliefs with natural parameter $\htheta$. In the generalized LMSR market with current market state $\theta$, the agent's optimal trade $\delta$ moves the state vector to
\begin{equation} \label{exp-update}
\theta + \delta = \frac{1}{1+a} \htheta + \frac{a}{1+a} \theta.
\end{equation}
\end{theorem}
\begin{proof}
The agent's optimal trade maximizes its expected utility, or equivalently the certainty equivalent. From Lemma~\ref{lem:exp-util} and that $T = C$, the agent maximizes
$$
\log a - T(\htheta - a\delta) + T(\htheta) - aT(\theta+\delta) + aT(\theta). 
$$
This objective is strictly concave, from the strict convexity of $T$. The optimum is therefore characterized by the first-order condition $\grad T(\htheta - a\delta) = \grad T(\theta+\delta)$. As the gradient map $\grad T$ is one-to-one, this is solved by equating the arguments, which leads to $\delta = (\htheta - \theta) / (1+a)$ and~(\ref{exp-update}).
\end{proof}

Note that as, $a$ tends to 0, we approach risk neutrality and the agent moves the share vector all the way to its private estimate $\htheta$. As $a$ grows larger (the agent grows more risk averse) the agent makes smaller trades to reduce it exposure, and the final state stays closer to the current state $\theta$. Update~(\ref{exp-update}) implies that, under the conditions of the theorem, a market that receives a sequence of myopic traders aggregates their natural parameters in the form of an exponentially weighted moving average. The final market estimates (i.e., prices) are obtained by applying $\grad T$ to this average.

\subsubsection*{Liquidity Adjustment}

In practice the centralized market maker allows itself some control over the \emph{liquidity} in the market, which captures how responsive prices are to trades. To adjust liquidity we consider the parametrized cost $C_{\lambda}(\theta)=\frac{1}{\lambda} C\left(\lambda\theta\right)$. Here $\lambda$ is construed as the \emph{inverse liquidity}, or price responsiveness. A larger setting of $\lambda$ means fewer shares need to be bought to reach the same prices.\footnote{The liquidity adjustment to the cost function takes the same form as the risk-aversion adjustment to the exponential utility in~(\ref{exp-util}). In convex analysis, this transformation is known as the perspective function~\citep[p.\ 90]{HiriartLe00}.}

In the context of the generalized LMSR we write $T$ rather than $C$ where $T$ is the log partition function, with liquidity-adjusted version $T_{\lambda}$. Let $\hmu$ be the agent's mean belief with corresponding natural parameter $\htheta = \grad T^{-1}(\hmu)$. Recall that a risk-neutral agent moves the share vector so that the prices match its mean parameter. Therefore, define the \emph{target shares} as $\ttheta = \grad T^{-1}_{\lambda}(\hmu)$. The target shares $\htheta$ and natural parameter $\ttheta$ are related by $\grad T(\htheta) = \grad T_{\lambda}(\ttheta) = \hmu$. In addition it is straightforward to check that $\grad T_{\lambda}(\ttheta) = \grad T(\lambda \ttheta)$ so we have
\begin{equation} \label{target}
\ttheta = \htheta / \lambda.
\end{equation}
Higher price responsiveness means fewer shares must be bought to make the market prices match the agent's expectation, so the natural parameter is scaled down accordingly. With a liquidity adjustment the analysis of Theorem~\ref{thm:exp-util} can be extended and yields the following result, where as before $\theta$ and $\mu$ are the market's outstanding shares and prices respectively.
\begin{corollary}
Under the conditions of Theorem~\ref{thm:exp-util} and an inverse liquidity of $\lambda$, the agent's optimal trade $\delta$ moves the state vector to
\begin{equation} \label{liqu-update}
\theta + \delta \:\:=\:\: \frac{\lambda}{\lambda+a} \ttheta + \frac{a}{\lambda+a} \theta
\:\:=\:\: \frac{\lambda}{\lambda+a} \grad T^{-1}_{\lambda}(\hmu) + \frac{a}{\lambda+a} \grad T^{-1}_{\lambda}(\mu).
\end{equation}
\end{corollary}
According to~(\ref{liqu-update}), as $\lambda$ grows large the agent moves the market state closer to the \emph{target shares}, rather than its true natural parameter. Note that the target shares themselves depend on $\lambda$ by~(\ref{target}), but the update can be directly written in terms of the agent's beliefs as in the right-hand side of~(\ref{liqu-update}).

\subsection{Repeated Trading and the Effective Belief}

In previous sections we have analyzed trader behavior assuming it is his first entry into the market. We now pose the question: how will a trader reason about a possible future investment when the trader holds an existing portfolio?
In the context of a trader possessing an exponential family belief together with exponential utility, we show that we can explicitly analyze how an agent incorporates an existing portfolio. The key conclusion is that a trader will reason about a future investment simply as though he had updated his belief and had no prior investment.

Suppose an exponential utility agent has exponential family belief parametrized by natural parameter $\hat{\theta}$. Based on this belief, let $\delta_1$ be the vector of shares the agent has purchased on first entry in the market.
On a subsequent entry into this market with market state $\theta'$, his optimal purchase $\delta_{2}^{*}$ is given by the solution of
$$
\arg\max_{\delta_{2}} \:\: \E_{x\sim p(x;\htheta)} U\left[ \la\delta_1+\delta_{2} ,\phi(x)\ra - C(\delta_1 + \theta) + C(\theta)- C(\delta_{2} + \theta') + C(\theta') \right].
$$
Then if $\theta'' = \hat{\theta}-a \delta_1  $ is the effective belief, the trader's optimal purchase is given by $\delta_{2} = (\theta'' - \theta') / (1+a)$, moving the share vector to $\theta' + \delta_{2} = \frac{1}{1+a} \theta'' + \frac{a}{1+a} \theta'$, which is a convex combination of the effective belief and the current market state.
\begin{theorem}\label{thm:reptrades}
Suppose an exponential utility maximizing trader with utility parameter $a$ who has belief $\hat{\theta}$ makes a purchase $\delta$ in a market. On subsequently re-entering the market, he will behave identically to an exponential utility maximizing trader with belief $\hat{\theta}-a \delta$ and no prior exposure in the market.
\end{theorem}
Theorem \ref{thm:reptrades} implies that financial exposure can be equivalently understood as changing the privately held beliefs.

\subsection{Equilibrium Market State for Exponential Utility Agents}

We have shown that every exponential-utility maximizing trader picks the share vector $\delta$ so that the eventual market state can be represented as a convex combination of the current market state and the natural parameter of his (exponential family) belief distribution. In this section we will compute the equilibrium state in an exponential family market with multiple such traders.

We draw a well-known result from game theory regarding the class of \emph{potential games}. We say a function $f(\vec{x})$ is at a \emph{local optimum} if changing any coordinate of $\vec{x}$ does not increase the value of $f$.
\begin{theorem}[\citet{monderer1996potential}]\label{thm:potl}
Let $U_{i}(\vec{\delta})$ be the utility function of the $i^{th}$ trader given strategies $\vec{\delta}= (\delta_{1},\ldots,\delta_{i},\ldots,\delta_{n})$. If there exists a potential function $\Phi(\vec{\delta})$ such that $$U_{i}(\vec{\delta})-U_{i}(\vec{\delta}_{-i},\delta_{i}')=\Phi(\vec{\delta})-\Phi(\vec{\delta}_{-i},\delta_{i}')$$
then $\vec{\delta}$ is a Nash equilibrium if and only if $\Phi(\vec{\delta})$ is at a local optimum.
\end{theorem}
In the exponential family market, the cost function $C$ is identical to the log partition function $T$ defined in (\ref{log-part}).  Let $\vec{\delta}$ be the matrix of share vectors purchased by every trader in the market at equilibrium. Let $\theta$ be the initial market state, $\htheta_{i}$ the natural parameter of trader $i$'s belief distribution and $a_{i}$ his risk aversion parameter.

Define a potential function as  $\Phi(\vec{\delta}) =  T\left(\theta+\sum_{i}\delta_{i}\right)+\sum_{i}\frac{1}{a_{i}}T(\htheta_{i}-a_{i}\delta_{i})$.

Rather than working directly with the utilities of every trader, we will work with the log of their utility values.\footnote{It is important to note that the potential function analysis still applies for any monotonically increasing transformation of the traders' utility functions.} Now the log-utility of trader $i$ is $$U_{i}(\vec{\delta})=- T(\theta+\sum_{j}\delta_{j}) + T(\theta+\sum_{j\neq i}\delta_{j}) - \frac{1}{a_{i}} T(\htheta_{i} - a_{i}\delta_{i}) + \frac{1}{a_{i}} T(\htheta_{i}).$$
We can now apply Theorem \ref{thm:potl}, hence the equilibrium state is obtained by jointly maximizing $\Phi(\vec{\delta})$ for each $\delta_{i}$:
\begin{eqnarray*}
\nabla_{\delta_{i}}\Phi(\vec{\delta})&=&\nabla T\left(\theta+\sum_{j=1}^{n}\delta_{i}\right)-\nabla T(\htheta_{i}-a_{i}\delta_{i}) = 0.
\end{eqnarray*}
This leads to the following expression for the final market state.
$$\theta+\sum_{j=1}^{n}\delta_{j}=
\frac{\theta+\sum_{i=1}^{n}\left(\frac{\htheta_{i}}{a_{i}}\right)}{1+\sum_{i=1}^{n}\frac{1}{a_{i}}}$$
We see that the equilibrium state is a convex combination of the initial market state and all agent beliefs, with the latter weighted according to risk tolerance.

\section{Budget-limited Aggregation} \label{sec:budgets}

In this section, we consider the evolution of the market state when traders are budget-limited.
We assume that the traders trade in multiple instances of the market. As before, the market price is interpreted as a probability density over the outcome space and the share vector as the natural parameter of an exponential family distribution. Consistent with the connections drawn in Section~\ref{sec:scoring} and throughout, we measure the error in prediction using the standard log loss.

We show that traders with faulty information can only impose a limited amount of additional loss to the market's prediction. Further, since informative traders experience an expected increase in budget,  they will eventually be unconstrained and allowed to carry out  unrestricted trades. Taken together, this means that while the market suffers limited damage from ill-informed traders, it is also able to make use of all the information from informative traders in the long run.
\paragraph{Budget-limited trades}

Let $\alpha$ be the budget of a trader in the market. Suppose that with infinite budget, the trader would have moved the market state from $\theta$ to $\htheta$, where $\htheta$ represents his true belief. Now suppose further that $\alpha<C(\htheta)-C(\theta)$; that is, the trader's budget does not allow for purchasing enough shares to move the market state to his belief. In this case, we want to budget-limit the trader's influence on the market state. 

Let the current market state be given by $\theta$ and let the final market state be $\theta'=\lambda\htheta+(1-\lambda)\theta$ where $\lambda=\min\left(1, \frac{\alpha}{C(\htheta) - C(\theta)}\right)$. The cost to the trader to move the market state from $\theta$ to $\theta'$ is at most his budget $\alpha$ and is called his \emph{budget-limited trade}.

\paragraph{Limited Damage}
We will now quantify the error in prediction that the market maker might have to endure as a result of ill-informed entities entering the market. We assume that these entities trade in multiple instances of the market; thus the exposure of the market maker is over several rounds. The log loss function for $\theta$ shares held is defined as $L(\theta,x)=-\log p(x;\theta)=C(\theta)-\la \theta\phi(x)\ra$.

\begin{lemma}\label{lem:damage}
The loss induced on the market by an uninformative trader is bounded by his initial budget.
\end{lemma}
\begin{proof}
First consider the change in budget of a trader $i$ over multiple rounds of the prediction market. Let his budget at rounds $t$ and $t-1$ be $\alpha_{i}^{t}$ and $\alpha_{i}^{t-1}$ respectively. The change in budget for trader $i$ moving the market state from $\theta$ to $\theta'$ with outcome $x^{t}$ is
\begin{eqnarray*}
\alpha_{i}^{t}-\alpha_{i}^{t-1}&=&C(\theta)-C(\theta')-(\theta-\theta')^{T}\phi(x^{t})\\
&=&L(\theta,x^{t})-L(\theta',x^{t})=\Delta_{i}^{t}
\end{eqnarray*}

Where $\Delta_{i}^{t}$ is called the myopic impact of a trader $i$ in round $t$. Thus, the myopic impact captures incremental gain in prediction due to the trader in a round and is equal to the change in his budget in that round.

Since the market evolves so that the budget of any trader never falls below zero, the total myopic impact in $T$ rounds caused due to trader $i$ is $\Delta_{i}:=\sum_{t=1}^{T}\Delta_{i}^{t}= \sum_{t=1}^{T} (\alpha_{i}^{t} - \alpha_{i}^{t-1}) = \alpha_{i}^{T} - \alpha_{i}^{0}\geq -\alpha_{i}^{0}$.
\end{proof}

An interesting aspect of Lemma~\ref{lem:damage} is that the log loss can be quantified in the same units as the traders' budgets.

\paragraph{Budget of Informative Traders}
We now characterize the expected change in budget for an informative trader. %
\begin{lemma}\label{thm:growth}
Let $\theta$ be the current market state. Suppose that an  informative trader with belief distribution  parametrized by $\htheta$  moves the market state to the budget-limited state $\theta'=\lambda \htheta + (1-\lambda) \theta$. Then, the expectation (over the trader's belief) of the trader's profit  is strictly positive whenever his budget is positive and his belief differs from the previous market position $\theta$. 
\end {lemma}
\begin{proof}
Let the cost function $C$ be equal to the log partition function $T$ of the belief distribution. The payoff is given by the sufficient statistics $\phi(x)$. Then, the trader's expected net payoff is given by
\begin{eqnarray*}
\E_{x\sim P_{\htheta}}[C(\theta)-C(\theta')-(\theta-\theta')\phi(x)]
&=& T(\theta)-\theta\nabla T(\htheta)-(T(\theta')-\theta' \nabla T(\htheta))\\
&=& D_{T}(\theta,\htheta)-D_{T}(\theta',\htheta) \geq \lambda D_{T}(\theta,\htheta)\geq0
\end{eqnarray*}
where $D_{T}(\cdot,\cdot)$ is the Bregman divergence based on $T$. The second to last inequality holds since $D_{T}(\theta',\htheta)$ is convex in $\theta'$ and we have:
\begin{eqnarray*}
D_{T}(\theta',\htheta)&=&D_{T}\left(\lambda\htheta+(1-\lambda)\theta,\htheta\right) \leq \lambda D_{T}( \htheta ,\htheta)+ (1-\lambda) D_{T}( \theta ,\htheta) 
= (1-\lambda) D_{T}( \theta ,\htheta)
\end{eqnarray*}¥
A trader who adjusts the market state may expect positive profit $\geq \lambda D_{T}(\theta,\htheta)$.
 \end{proof}

We note one important aspect of Lemma~\ref{thm:growth}: the expectation is taken with respect to each trader's belief at the time of trade, rather than with respect to the true distribution. This is needed because we have made no assumptions about the optimality of the traders' belief updating procedure. If we assume that the traders' belief formation is optimal, then this growth result will extend to the true distribution as well.

Given a continuous density the probability a trader will form exactly the same beliefs as the current market position is $0$, and thus, each trader will have positive expected profit on almost all sequences of observed samples and beliefs. 
This result suggests that, eventually, every informative trader will have the ability to influence the market state in accordance with his beliefs, without being budget limited.

Notice that Lemma~\ref{thm:growth} only required that the market state to which the trader moves be representable as a convex combination of the current market state and his belief. This means that the result holds for exponential utility traders aiming to maximize their utility by Theorem~\ref{thm:exp-util}. In this case, the trader who moves the market state can expect his profit to be positive and at least $\frac{1}{a}D_{T}(\theta,\htheta)$ where $a$ is the exponential utility parameter. When the cost function is adjusted to $C_{\lambda}$ with an inverse liquidity parameter $\lambda$ as in Section \ref{sec:exputil},  the trader receives an expected payoff of at least $\frac{1}{a} D_{T}(\lambda\theta,\htheta)$.

\bibliographystyle{acmsmall}
\bibliography{maxentpredmkts}
\end{document}